\RequirePackage{scrlfile}
\PreventPackageFromLoading{subfig}
\documentclass{article}

\usepackage[final]{changes}

\pdfoutput=1
\usepackage{subcaption}
\usepackage{svg}
\usepackage{hyperref}

\PassOptionsToPackage{numbers, compress}{natbib}
\usepackage[final]{neurips_2019}
\setcitestyle{numeric}

\usepackage[english]{babel}
\usepackage[utf8]{inputenc} 
\usepackage[T1]{fontenc}    
\usepackage{booktabs}
\usepackage{tabularx}
\usepackage{multirow}
\usepackage{times}
\usepackage{hyperref}
\usepackage{url}
\usepackage{amssymb}
\usepackage{amsmath}
\usepackage{graphicx}
\usepackage{sidecap, caption}
\usepackage{courier}
\usepackage{color}
\usepackage{wrapfig}

\usepackage{amsthm}
\usepackage{booktabs, topcapt}
\usepackage{nicefrac}       
\usepackage{microtype}      
\usepackage{amsfonts}       

\newcommand{\be}{\begin{eqnarray} \begin{aligned}}
\newcommand{\ee}{\end{aligned} \end{eqnarray} }
\newcommand{\benn}{\begin{eqnarray*} \begin{aligned}}
\newcommand{\eenn}{\end{aligned} \end{eqnarray*} }
\newcommand{\vx}{{\mathbf x}}

\newcommand{\vz}{{\mathbf z}}

\newcommand{\kl}{{D_{KL}}}
\newcommand{\dec}{{p_{\theta} (\vx|\vz)}}
\newcommand{\enc}{{q_{\phi} (\vz|\vx)}}

\newcommand{\lld}{{p_{\theta} (\vx)}}
\newcommand{\pz}{{p(\vz)}}
\newcommand{\px}{{p_{\theta}(\vx)}}
\newcommand{\qd}{{{q}(\vx)}}
\newcommand{\qz}{{q_{\phi}(\vz)}}
\newcommand{\iq}{{I_{q}(X;Z)}}
\newcommand{\eq}{{\mathbb{E}_{q}}}

\newcommand{\bfeps}{{\mathbf{\epsilon}}}
\newcommand{\ellprime}{{\ell^{\prime}}}

\newcommand{\dmax}{{d_{\max}}}
\newcommand{\rad}{{r}}
\newcommand{\iidsim}{{\overset{iid}{\sim}}}

\makeatletter
\def\thm@space@setup{%
  \thm@preskip=\parskip \thm@postskip=0pt
}
\makeatother

\newtheorem{theorem}{Theorem}[section]
\newtheorem*{theorem*}{Theorem}
\newtheorem{lemma}[theorem]{Lemma}

\newenvironment{definition}[1][Definition]{\begin{trivlist}
\item[\hskip \labelsep {\bfseries #1}]}{\end{trivlist}}

\title{Exact Rate-Distortion in Autoencoders via Echo Noise}
\author{%
    Rob Brekelmans, Daniel Moyer, Aram Galstyan, Greg Ver Steeg \\
    Information Sciences Institute \\
    University of Southern California\\
    Marina del Rey, CA 90292 \\
    \texttt{brekelma, moyerd@usc.edu; galstyan, gregv@isi.edu} \\
}

\begin{document}

\maketitle 

\begin{abstract}
Compression is at the heart of effective representation learning. However, lossy compression is typically achieved through simple parametric models like Gaussian noise to preserve analytic tractability, and the limitations this imposes on learning are largely unexplored.  Further, the Gaussian prior assumptions in models such as variational autoencoders (VAEs) provide only an upper bound on the compression rate in general.  We introduce a new noise channel, \emph{Echo noise}, that admits a simple, exact expression for mutual information for arbitrary input distributions.  The noise is constructed in a data-driven fashion that does not require restrictive distributional assumptions.  With its complex encoding mechanism and exact rate regularization, Echo leads to improved bounds on log-likelihood and dominates $\beta$-VAEs across the achievable range of rate-distortion trade-offs. Further, we show that Echo noise can outperform flow-based methods without the need to train additional distributional transformations.

\end{abstract}

\section{Introduction}

Rate-distortion theory provides an organizing principle for representation learning that is enshrined in machine learning as the Information Bottleneck principle~\citep{tishby2000information}. 
The goal is to compress input random variables $X$ into a representation $Z$ with mutual information rate $I(X;Z)$, while minimizing a distortion measure that captures our ability to use the representation for a task.  For the rate to be restricted, some information must be lost through noise. Despite the use of increasingly complex encoding functions via neural networks, simple noise models like Gaussians still dominate the literature because of their analytic tractability. Unfortunately, the effect of these assumptions on the quality of learned representations is not well understood. 

The Variational Autoencoding (VAE) framework \citep{kingma2013auto, rezende2014stochastic} has provided the basis for a number of recent developments in representation learning \citep{informationdropout, chen2018isolating, chen2016variational, higgins2017, kim2018disentangling,  tschannen2018recent}.  While VAEs were originally motivated as performing posterior inference under a generative model, several recent works have viewed the Evidence Lower Bound objective as corresponding to an unsupervised rate-distortion problem \citep{informationdropout, alemi2018fixing, rezende2018taming}.  From this perspective, reconstruction of the input provides the distortion measure, while the KL divergence between encoder and prior gives an upper bound on the information rate that depends heavily on the choice of prior \citep{alemi2018fixing,  rosca2018distribution, tomczak2017vae}.

In this work, we deconstruct this interpretation of VAEs and their extensions.  Do the restrictive assumptions of the Gaussian noise model limit the quality of VAE representations?  Does forcing the latent space to be independent and Gaussian constrain the expressivity of our models?  We find evidence to support both claims, showing that a powerful noise model can achieve more efficient lossy compression and that relaxing prior or marginal assumptions can lead to better bounds on both the information rate and log-likelihood. 

The main contribution of this paper is the introduction of the Echo noise channel, a powerful, data-driven improvement over Gaussian channels whose compression rate can be precisely expressed for arbitrary input distributions.  Echo noise is constructed from the empirical distribution of its inputs, allowing its variation to reflect that of the source (see Fig.~\ref{fig:example}). We leverage this relationship to derive an analytic form for mutual information that avoids distributional assumptions on either the noise or the encoding marginal.  Further, the Echo channel avoids the need to specify a prior, and instead implicitly uses the optimal prior in the Evidence Lower Bound.  This marginal distribution is neither Gaussian nor independent in general.

\begin{figure}
  \centering
  \includegraphics[scale=.4]{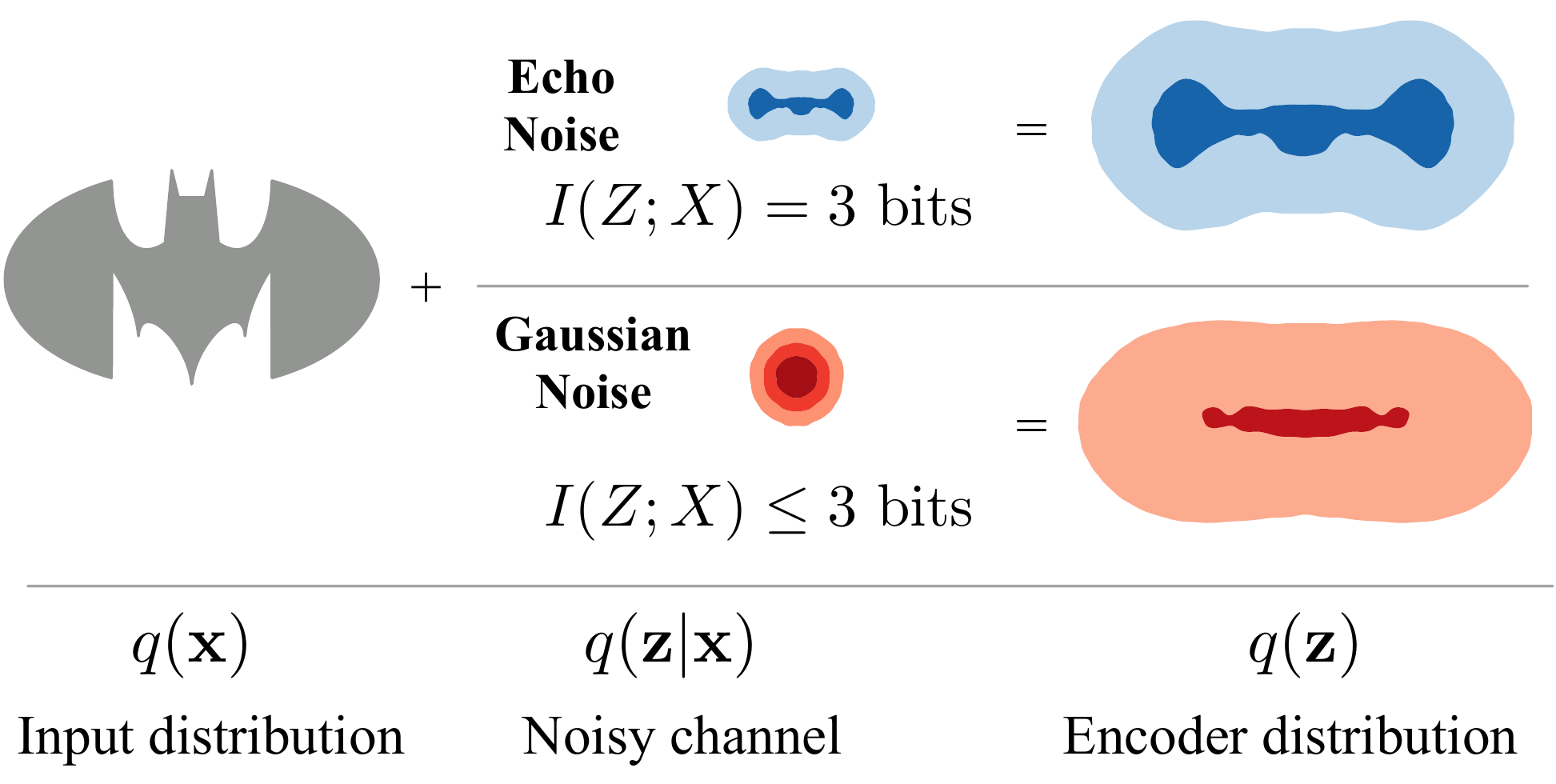} 
  \caption{For a noisy channel characterized by $\vz = \vx + s \bfeps$, we compare drawing the noise, $\bfeps$, from a Gaussian distribution (as in VAEs) or an Echo distribution.}
  \label{fig:example}
\end{figure}

After introducing the Echo noise channel and an exact characterization of its information rate in Sec.~\ref{sec:echo}, we proceed to interpret Variational Autoencoders from an encoding perspective in Sec.~\ref{sec:vae}.  We formally define our rate-distortion objective in Sec.~\ref{sec:problem}, and draw connections with recent related works in Sec.~\ref{sec:related}.  Finally, we report log likelihood results, \replaced{ visualize the space of compression-reconstruction trade-offs, and evaluate disentanglement in Echo representations in Sec.~\ref{sec:results}.}{ and visualize the space of possible compression-reconstruction trade-offs in Sec.~\ref{sec:results}}


\section{Echo Noise} \label{sec:echo} 

To avoid learning representations that memorize the data, we would like to constrain the mutual information between the input $X$ and the representation $Z$. Since we have freedom to choose how to encode the data, we can design a noise model that facilitates calculating this generally intractable quantity.


The Echo noise channel has a shift-and-scale form that mirrors the reparameterization trick in VAEs.  Referring to the observed data distribution as $\qd$, with $\vz \in \mathbb R^{d_z}, \vx \in \mathbb R^{d_x}$, we can define the stochastic encoder $q_\phi(\vz |\vx)$ using: 
\be
\vz   & =   f (\vx) +  S(\vx) \bfeps \label{eq:reparam}
\ee
For brevity, we omit the subscripts that indicate that the functions $f: \mathbb R^{d_x} \rightarrow \mathbb R^{d_z}$ and matrix function $S:  \mathbb R^{d_x} \rightarrow \mathbb R^{d_z} \times  \mathbb R^{d_z}$ depend on neural networks parameterized by $\phi$. All that remains to specify the encoder is to fix the distribution of the noise variable, $q(\bfeps)$. For VAEs, the noise is typically chosen to be Gaussian, $\bfeps \sim \mathcal N(0, \mathbb I_{d_z})$. \footnote{Our approach is also easily adapted to multiplicative noise, such as in \citet{informationdropout}.}


With the goal of calculating mutual information, we will need to compare the marginal entropy $H(Z)$, which integrates over samples $\vx$, and the conditional entropy $H(Z|X)$, whose stochasticity is only due to the noise for deterministic $f(\vx)$ and $S(\vx)$. 
The choice of noise will affect both quantities, and our approach is to relate them by enforcing an equivalence between the distributions $q(\vz)$ and $q(\bfeps)$.  

Since $q(\vz) = \int \enc \qd d\vx$ is defined in terms of the source, we can also imagine constructing the noise in a data-driven way.  For instance, we could draw $\bfeps = f(x'), x' \, \iidsim \qd$ in an effort to make the noise match the channel output.  However, this changes the distribution of $Z$ and the noise would need to be updated to continue resembling the output.  

Instead, by iteratively applying Eq.~\ref{eq:reparam}, we can guarantee that the noise and marginal distributions match in the limit. Using superscripts to indicate iid samples $\vx^\ell \, \iidsim \qd$, we draw $\bfeps$ according to:
\begin{align}
\bfeps &= f(\vx^{0}) + S(\vx^{0})\bigg( f(\vx^{1})  +  S(\vx^{1}) \Big(f(\vx^2) + S(\vx^2) \big(... \nonumber  \\ 
& = f(\vx^{0}) + S(\vx^{0}) f(\vx^{1})  +  S(\vx^{0}) S(\vx^{1}) f(\vx^{2}) ... \label{eq:noise_poly}  
\end{align}


\added{Echo noise is thus constructed using an infinite sum over attenuated ``echoes'' of the transformed data samples}.  This can be written more compactly as follows. 

\begin{definition}[Definition: Echo Noise]  The Echo noise distribution $ E(f(\vx), S(\vx), q(\vx))$ is defined for functions $f, S,$ and probability density function $q$ over $\vx \in \mathbb R^{d_x}$, by sampling according to the following procedure. 
\be
\bfeps  & = \sum \limits_{\ell=0}^{\infty} \left(\prod \limits_{\ellprime= 1}^{\ell} S(\vx^{\ellprime})  \right) f(\vx^{\ell}), \qquad \vx^{\ell}\,  \iidsim  \qd \label{eq:noise} 
\ee 
\end{definition}

Although the noise distribution may be complex, it has the interesting property that it exactly matches the eventual output marginal $\qz$.  

\begin{lemma}[Echo noise matches channel output]\label{equivalence}
If $\bfeps \sim Echo(f(\vx), S(\vx), q(\vx))$ and $\vz=f (\vx) +  S(\vx) \bfeps$, then $\vz$ has the same distribution as $\bfeps$.
\end{lemma}
We can observe this relationship by simply re-labeling the sample indices in the expanded expression for the noise in Eq.~\ref{eq:noise_poly}.  In particular, the training example that we condition on in Eq.~\ref{eq:reparam} corresponds to the first sample $\vx^0$ in a draw from the noise.  This equivalence is the key insight leading to an exact expression for the mutual information: 


\begin{theorem}[Echo Information]\label{thm:mi} For any source distribution $\qd$, and a noisy channel defined by Eq.~\ref{eq:reparam} that satisfies \ref{condition}, the mutual information is as follows:
\be\label{eq:mi}
I(X;Z) = - \mathbb{E}_\vx \log  |\det S(\vx)| 
\ee
\end{theorem}
\begin{proof}
\replaced{We start by expanding the definition of mutual information in terms of entropies.   Since $f(\vx)$ and $S(\vx)$ are deterministic, we treat them as constants after conditioning on $X=\vx$.  The stochasticity underlying $H(Z|X=\vx)$ is thus only due to the random variable $\bfeps$.} { We start by expanding the definition of mutual information in terms of entropies, reiterating that the stochasticity underlying $H(Z|X=\vx)$ after conditioning on $X=\vx$ is due only to the random variable $\bfeps$. }
\be 
I(X;Z) &= H(Z) - H(Z|X) \nonumber \\
&=  H(Z)- \mathbb{E}_{\vx} \, H( f(\vx) + S(\vx) \mathcal E \, |  \, X= \vx ) \nonumber \\
&=  H(Z) -  \mathbb{E}_{\vx} \, H( S(\vx)  \mathcal E \, |  \, X= \vx )\nonumber \\
& = H(Z) -   H(\mathcal{E}) - \mathbb{E}_{\vx} \log  |\det S(\vx)| \nonumber \\
& = - \mathbb{E}_{\vx} \log  |\det S(X)| 
\ee
  \replaced{We have used the translation invariance of differential entropy in the third line, and the scaling property in the fourth line \cite{cover}.  The entropy terms cancel as a result of Lemma~\ref{equivalence}. } {  Since $f(\vx)$ and $S(\vx)$ are deterministic, we treat them as constants inside the expectation.  We use the translation invariance of differential entropy in the third line, and the scaling property in the fourth line \cite{cover}. Finally, we use Lemma~\ref{equivalence} to cancel out the entropy terms.}
\end{proof}

 In this work, we consider only diagonal $S(\vx) \equiv \mbox{diag}(s_1(\vx), \ldots, s_{d_z}(\vx))$ as is typical for VAEs, so that the determinant in Eq.~\ref{eq:mi} simplifies as $I(X;Z) = - \sum_j \mathbb{E}_{\vx} \log |s_j(\vx)| = \sum_j I(X;Z_j).$  

Finally, we note that the noise distribution $q(\bfeps)$ is only defined implicitly through a sampling procedure. For this to be meaningful, we must ensure that the infinite sum converges. 
\begin{lemma}\label{condition}
The infinite sum in Eq.~\ref{eq:noise} converges, and thus Echo noise sampling is well-behaved, if  $\forall \vx, \, \exists M $ s.t. $ |f(\vx)| \leq M $ and $\rho(S(\vx)) < 1$, where $\rho$ is the spectral radius.
\end{lemma}

In App.~\ref{app:implementation}, we discuss several implementation choices to guarantee that these conditions are met and that Echo noise can be accurately sampled using a finite number of terms.  This \replaced{is}{ . can be} particularly difficult in the high noise, low information regime, as zero mutual information ($s_j(\vx)) = 1 \, \forall \, \vx, j$) would imply an infinite amount of noise.  To avoid this issue and ensure precise sampling, we \added{ clip the magnitude of $s_j(\vx)$ so that, for a given $M$ and number of samples, the sum of remainder terms is guaranteed to be within machine precision.}  This imposes a lower bound on the achievable rate across the Echo channel, which depends on the number of terms considered and can be tuned by the practitioner. 

\subsection{Properties of Echo Noise} \label{sec:echo_properties} 



We can visualize applying Echo noise to a complex input distribution in Fig.~\ref{fig:example}, using the identity transformation $f(\vx) = \vx$ and constant noise scaling $s_j(\vx) = .5$.  Here, we directly observe the equivalence of the noise and output distributions.  Further, the data-driven nature of the Echo channel means it can leverage the structure in the (transformed) input to destroy information in a more targeted way than spherical Gaussian noise. 

In particular, Echo's ability to add noise that is correlated across dimensions distinguishes it from common diagonal noise models. It is important to note that the noise still reflects the dependence in $f(\vx)$ even when $S(\vx)$ is diagonal.  In fact, we show in App.~\ref{app:tc} that $TC(Z) = TC(Z|X)$ for the diagonal case, where total correlation measures the divergence from independence, e.g. $TC(Z|X) = \kl [ q(\vz|\vx) || \prod q(z_j|\vx)]$ \citep{watanabe}.

In the setting of learned $f(\vx)$ and $S(\vx)$, notice that the noise depends on the parameters.  This means that training gradients are propagated through $\bfeps$, unlike traditional VAEs where $q(\bfeps)$ is fixed.   This may be a factor in improved performance:  data samples are used as both signal and noise in different parts of the optimization, leading to a more efficient use of data.


Finally, the Echo channel fulfills several of the desirable properties that often motivate Gaussian noise and prior assumptions.  Eqs.~\ref{eq:reparam} and \ref{eq:noise} define a simple sampling procedure that only requires a supply of iid samples from the input distribution.  It is easy to sample both the noise and conditional distributions for the purposes of evaluating expectations, while Echo also provides a natural way to sample from the true encoding marginal $\qz$ via its equivalence with $q(\bfeps)$.  While we cannot evaluate the density of a given $\vz$ under $\enc$ or $\qz$, as might be useful in importance sampling \cite{burda2015importance}, we can characterize their relationship \textit{on average} using the mutual information in Eq.~\ref{eq:mi}.  These ingredients make Echo noise useful for learning representations within the autoencoder framework.

\section{Lossy Compression in VAEs} \label{sec:vae}
Variational Autoencoders (VAEs) \citep{kingma2013auto,rezende2014stochastic} seek to maximize the log-likelihood of data under a latent factor generative model defined by $p_\theta(\vx,\vz) =  \pz p_\theta(\vx|\vz)$, where $\theta$ represents parameters of the generative model decoder and $p(\vz)$ is the prior distribution over latent variables. However, maximum likelihood is intractable in general due to the difficult integral over $Z$, $\log \px = \log \int \pz \dec d\vz$.

To avoid this problem, VAEs introduce a variational distribution, $q_\phi(\vz|\vx)$, which encodes the input data $\qd$ and approximates the generative model posterior $p_{\theta}(\vz|\vx)$.
This leads to the tractable (average) Evidence Lower Bound (ELBO) on likelihood: 
\be\label{eq:elbo}
\eq \log \px & \geq \eq \log \px - \kl [ \enc || p_{\theta}(\vz|\vx) ] \\
& = \eq \log \dec  - \kl [ \enc   || p(\vz) ]
\ee

The connection between VAEs and rate-distortion theory can be seen using a decomposition of the KL divergence term from \citet{elbosurgery}.  
\begin{align}\label{eq:mi_ub2}
\kl [ \enc   || p(\vz) ] &= \kl [ \enc   || \qz ] + \kl[\qz || p(\vz)] \\ \nonumber 
&\geq \kl [ \enc   || q_\phi(\vz)) ] =  \iq 
\end{align}
This decomposition lends insight into the orthogonal goals of the ELBO regularization term.  The mutual information $\iq$ encourages lossy compression of the data into a latent code, while the marginal divergence enforces consistency with the prior.  The non-negativity of the KL divergence implies that each of these terms detracts from our likelihood bound.

Similarly, we observe that $\kl [ \enc   || p(\vz) ]$ gives an upper bound on the mutual information, with a gap of $\kl [ \qz || p(\vz)]$.   From this perspective, a static Gaussian prior can be seen a particular and possibly loose marginal approximation \citep{ alemi2018fixing, gao2018auto, rosca2018distribution}.  The true encoding marginal $\qz$ provides the unique, optimal choice of prior and leads to a tighter bound on the likelihood:
\be\label{eq:rd_bound}
\eq \log \lld & \geq  \eq \log p_\theta(\vx|\vz)  - \iq \label{eq:tight_elbo}
\ee


Our exact expression for the mutual information over an Echo channel provides the first general method to directly optimize this objective.  This corresponds to adaptively setting $\pz$ equal to $\qz$ throughout training, so that Eq.~\ref{eq:tight_elbo} can be seen as bounding the likelihood under the generative model $p(\vx) = \int \qz \dec d{\vx}$.






\subsection{Rate-Distortion Objective}\label{sec:problem}
While the VAE is motivated as performing amortized inference of the latent variables in a generative model, the prior is rarely leveraged to encode domain-specific structure.  Further, we have shown that enforcing prior consistency can detract from likelihood bounds.

We instead follow \citet{alemi2018fixing} in advocating that representation learning be motivated from an encoding perspective using rate-distortion theory.  In particular, \added{we choose reconstruction under the generative model as the distortion measure $d(\vx, \vz) = - \log \dec$}, and study the following optimization problem:
\be
\max_{\theta, \phi} \mathbb E_{q_\phi} \log p_\theta(\vx|\vz)- \beta \iq \label{eq:problemdef}
\ee


While this resembles the $\beta$-VAE objective of \citet{higgins2017}, we highlight two notable distinctions.  First, treating $\iq$ rather than the upper bound $\kl[\enc||\pz]$ avoids the need to specify a prior and facilitates a direct interpretation in terms of lossy compression.  Further, the $\beta$ parameter is naturally interpreted as a Lagrange multiplier enforcing a constraint on $\iq$.  The special choice of $\beta=1$ gives a bound on log-likelihood according to Eq.~\ref{eq:rd_bound}, which we use to compare results across methods in Sec.~\ref{sec:results}.  We direct the reader to App.~\ref{app:rd} for a more formal treatment of rate-distortion.

\section{Related Work}\label{sec:related}
\textbf{Rate-Distortion Theory:  } 
A number of recent works have made connections between the Evidence Lower Bound objective and rate-distortion theory \citep{informationdropout, alemi2018fixing, infolowerbounds, rezende2018taming}, with \replaced{the average distortion corresponding to the cross entropy reconstruction loss as above.}{reconstruction of the input providing the distortion measure  $d(\vx, \vz) = - \log \dec$}.  In particular, \citet{alemi2018fixing} consider the following upper and lower bounds on the mutual information $\iq$:
\begin{align*}
H - D = H_{q}(X) +  \eq \log \dec \leq \iq \leq \kl [ \enc || r(\vz) ] = R 
\end{align*}
With the data entropy as a constant, minimizing the cross entropy distortion corresponds to the variational information maximization lower bound of \citet{barber2003algorithm}.  The upper bound matches the decomposition in Eq. \ref{eq:mi_ub2} for the generalized choice of marginal $r(\vz)$.  Several recent works have also considered `learned priors' or flow-based density estimators \cite{alemi2018fixing, chen2016variational, tomczak2017vae} that seek to reduce the marginal divergence by approximating $\qz$ (see below).  Using this upper bound on the rate term, \citet{alemi2018fixing} and \citet{rezende2018taming} obtain objectives similar to Eq.~\ref{eq:problemdef}.

Existing models are usually trained with a static $\beta$ \cite{alemi2018fixing, higgins2017} or a heuristic annealing schedule \cite{bowman2016, burgess2018understanding}, which implicitly correspond to constant constraints (see App.\ref{app:rd}).  However, setting target values for either the rate or distortion remains an interesting direction for future discussion. \citet{rezende2018taming} view the distortion as an intuitive quantity to specify in practice, while \citet{zhao2018lagrangian} train a separate model to provide constraint values.  As both works show, specifying a constant and optimizing the Lagrange multiplier $\beta$ with gradient descent can lead to improved performance.

\textbf{Mutual Information in Unsupervised Learning:  } \label{sec:compare}
A number of recent works have argued that the maximum likelihood objective may be insufficient to guarantee useful representations of data \cite{alemi2018fixing, infovae}.  In particular, when paired with powerful decoders that can match the data distribution, VAEs may learn to completely ignore the latent code \cite{bowman2016,chen2016variational}.  

To rectify these issues, a commonly proposed solution has been to add terms to the objective function that maximize, minimize, or constrain the mutual information between data and representation \citep{alemi2018fixing, braithwaite2018bounded, phuong2018mutual, infovae, zhao2018lagrangian}.  However, justifications for these approaches have varied and numerous methods have been employed for estimating the mutual information.   These include sampling \citep{phuong2018mutual}, indirect optimization via other divergences \citep{infovae}, mixture entropy estimation \citep{kolchinsky2017nonlinear}, learned mixtures \citep{tomczak2017vae}, autoregressive density estimation \citep{alemi2018fixing}, and a dual form of the KL divergence \citep{belghazi2018mine}. \added{ \citet{poole2019variational} provide a thorough review and analysis of variational upper and lower bounds on mutual information, although recent results have shown limits on our ability to construct high confidence estimators directly from samples \cite{mcallester2018formal}.   Echo notably avoids this limitation by providing an analytic expression for the rate whenever the representation is sampled according to Eq. \ref{eq:noise}.} 
\replaced{Among the approaches above,}{Among these approaches}, the InfoVAE model of \citet{infovae} provides a potentially interesting comparison with our method.  The objective adds a parameter $\lambda$ to more heavily regularize the marginal divergence and a parameter $\alpha$ to control mutual information.  However, since $\kl[\qz||\pz]$ is intractable, the Maximum Mean Discrepancy (MMD) \cite{gretton2012kernel} between the encoding outputs and a standard Gaussian is used as a proxy.  For the choice of $\lambda = 1000$ (as in the original paper) and $\alpha = 0$ (no information preference), the objective simplifies to:
\benn
\mathcal{L}_{\mathrm{InfoVAE}} & = \mathcal{L}_{ELBO} - 999 \, D_{\mathrm{MMD}} [ \qz || \pz]
\eenn



The sizeable MMD penalty encourages $ \qz \approx \pz$, so that $\kl [\enc || \pz] \approx \kl [\enc || \qz] = \iq$.  Thus, the KL divergence term in the ELBO should more closely reflect a mutual information regularizer, facilitating comparison with the rate in Echo models.   



Flow models, which evaluate densities on simple distributions such as Gaussians but apply complex transformations with tractable Jacobians, are another prominent recent development in unsupervised learning \citep{germain2015made, kingma2016improved, papamakarios2017masked, rezende2015variational}.  Flows can be used both as an encoding mechanism and marginal approximation for our purposes.  In particular, Inverse Autoregressive Flow \cite{kingma2016improved} can be seen as transforming the output of a Gaussian noise channel into an approximate posterior sample using a stack of autoregressive networks.  Masked Autoregressive Flow \cite{papamakarios2017masked} models a similar transformation with computational tradeoffs suited for density estimation, mapping latent samples to high probability under a Gaussian base distribution to approximate $\qz$.




Finally, the VampPrior \cite{tomczak2017vae} may also be used as a marginal approximation, modeling $\qz$ using a mixture distribution $\frac{1}{K} \sum_{k} q_{\phi}(\vz|\mathbf{u}_k)$ evaluated on a set of `pseudo-inputs'  $\mathbf{u}_k \in \mathbb{R}^{d_x}$ learned by backpropagation.
\section{Results}\label{sec:results}
In this section, we would ideally like to quantify the impact of three key elements of the Echo approach: a data-driven noise model, exact rate regularization throughout training, and a flexible marginal distribution.  In App.~\ref{app:marginals}, we observe that the dimension-wise marginals learned by Echo appear Gaussian despite our lack of explicit constraints.  However, the \textit{joint} marginal over $\qz$ (or equivalently $q(\epsilon)$) may still have a complex dependence structure, which is not penalized for deviating from independence or Gaussianity.  We calculate a second-order approximation of total correlation in App.~\ref{app:tc} to confirm that this noise is indeed dependent across dimensions.  

\subsection{ELBO Results}\label{sec:elboresults}
\begin{table}[t]
\centering
\caption{Test Log Likelihood Bounds}
\label{tab:elbos}
\small
\begin{tabular}{lcccc}
 & \multicolumn{4}{c}{Binary MNIST} \\ \cmidrule(r){2-5}
Method & Rate & Dist & -ELBO & $\sigma$ \\ \midrule
Echo & 26.4 & 62.4 & \textbf{88.8} & \replaced{.18}{ .31} \\ \midrule  
VAE & 26.2 & 63.6 & 89.8 & .18 \\ 
InfoVAE & 26.0 & 64.0 & 90.0 & .14  \\ 
VAE-MAF & 26.1 & 63.7 & 89.8  & .15 \\  
VAE-Vamp & 26.3 & 63.0 & 89.3 & .19 \\\midrule  
IAF-Prior & 26.5 & 63.5 & 90.0 & .13 \\     
IAF+MMD & 26.3 & 63.6 & 90.1 & \replaced{.15}{.22}   \\ 
IAF-MAF & 26.4 & 63.6 & 89.9 & .18  \\ %
IAF-Vamp & 26.4 & 62.8 & 89.2 & .18  \\ 
\bottomrule
\end{tabular}\hspace*{-0.1cm}%
\begin{tabular}{cccc} 
 \multicolumn{4}{c}{Omniglot} \\ \cmidrule(r){1-4}
 Rate & Dist &  -ELBO & $\sigma$  \\ \midrule
 30.2 & 84.4 & \textbf{114.6} & \replaced{.30}{ .43}\\ \midrule 
 30.5 & 86.5 & 117.0 & .44 \\ 
30.3 & 87.3 & 117.6 & .51 \\  
30.5 & 86.4 & 116.9 & .31  \\ 
30.8 & 84.3 & 115.1 & .28 \\ \midrule 
 30.5 & 86.7 & 117.2 & \replaced{.36}{ .23} \\ 
 30.7 & 86.4 & 117.1 & \replaced{.28}{ .15} \\ 
 30.6 & 86.5 & 117.1 & .24 \\ 
 30.4 & 85.0 & 115.4 & .20 \\ 
\bottomrule 
\end{tabular}\hspace*{-0.1cm}
\begin{tabular}{cccc}
\multicolumn{4}{c}{Fashion MNIST} \\ \cmidrule(r){1-4}
Rate &  Dist & -ELBO & $\sigma$ \\ \midrule
16.6 & 218.3 & 234.9 & .21 \\ \midrule 
15.7 & 219.3 & 235.0 & .10  \\ 
15.6 & 219.5 & 235.1 & .10 \\ 
15.7 & 219.3 & 234.9  & .14 \\ 
 15.9 & 218.5 & 234.4 & .08 \\ \midrule 
15.8 & 219.1 & 234.9 & .10  \\ 
15.7 & 219.2 & 234.9 & .13 \\ 
15.8 & 219.1 & 234.9 & .14 \\ 
16.0 & 218.3 & \textbf{234.3} & .16 \\  
\bottomrule
\end{tabular}%
\begin{tabular}{c} 
Params \\ \cmidrule(r){1-1}
($\cdot 10^6$) \\ \midrule
1.40 \\  \midrule 
1.40 \\ 
1.40 \\
3.12 \\
1.99 \\ \midrule
3.12 \\ 
3.12 \\ 
4.84 \\
3.71 \\  
\bottomrule
\end{tabular}
\end{table}
We proceed to analyse the log-likelihood performance of relevant models on three image datasets: static Binary MNIST \citep{salakhutdinov2008quantitative}, Omniglot \citep{lake2015human} as adapted by \citet{burda2015importance}, and Fashion MNIST (fMNIST) \citep{xiao2017}.   All models are trained with 32 latent variables using the same convolutional architecture as in \citet{alemi2018fixing} except with ReLU activations.  We trained using Adam optimization for 200 epochs, with an initial learning rate of 0.0003 decaying linearly to 0 over the last 100 epochs.  \deleted{We also found it necessary to use a smaller bandwidth than in \citet{infovae} to enforce prior consistency with the MMD penalty. See App.\ref{app:experiments} for additional details.}

 Table \ref{tab:elbos} shows negative test ELBO values, with the rate column reported as the appropriate upper bound for comparison methods. \added{Results are averaged from ten runs of each model after removing the highest and lowest outliers.}
 We compare Echo against diagonal Gaussian noise and IAF encoders, each with four marginal approximations: a Gaussian prior with and without the MMD penalty (e.g. \textit{IAF-Prior}, \textit{IAF+MMD}), MAF \cite{papamakarios2017masked}, and VampPrior \cite{tomczak2017vae}.  Note that \textit{VAE} is still used to denote the Gaussian encoder when paired with a different marginal (e.g. \textit{VAE-Vamp}).

We find that the Echo noise autoencoder obtains improved likelihood bounds on Binary MNIST and Omniglot, with competitive results on fMNIST.  We emphasize that Echo \replaced{achieves this performance}{ approach provides these gains} with significantly fewer parameters than comparison methods.  IAF and MAF each require training an additional autoregressive model with size similar to the original network, while the VampPrior uses 750 learned pseudoinputs of the same dimension as the data.  Although Echo involves \replaced{special}{ extra} computation to construct the noise for each training example, it has the same number of parameters as a standard VAE and runs in approximately the same wall clock time.

\replaced{We observe only minor differences based on the choice of encoding mechanism, which is somewhat surprising given the additional expressivity of the IAF transformation.  The benefit of the flow transformations may be more readily observed on more difficult datasets or with more advanced architecture tuning \cite{kingma2016improved}.   

We do find that a more complex marginal approximation can help performance.  Although we see minimal gains from the MMD penalty and MAF marginal, the VampPrior bridges much of the performance gap with Echo noise.  Recall that a learned prior can help ensure a tight rate bound while providing flexibility to learn a more complex marginal (in this case, a mixture model).   However, the relative contribution of these effects is difficult to decouple.   Echo instead provides both an exact rate and an adaptive prior by directly linking the choice of encoder and marginal. }{ \\ \\ The Echo encoder is evidently more flexible than a Gaussian noise channel, and Echo dominates VAEs across datasets with up to six nats of improvement.  The additional expressivity of the IAF transformation is necessary to approach Echo performance, although IAF is not directly interpretable as a noise model and may still be paired with an inexact marginal. }

\replaced{}{We can gain insight into the benefit of exact rate regularization for Gaussian priors using the MMD penalty to enforce $q(\vz) \approx \mathcal{N}(0, \mathbb{I})$.  This leads to slightly improved performance for InfoVAE, although IAF encoders do not appear to be sensitive to this change.  It may be surprising that additional regularization can lead to better ELBOs, but we argue MMD can guide optimization toward solutions that make efficient use of the rate penalty $\kl [\enc||\pz]$.  The data processing inequality \cite{cover} states that only those nats of mutual information can be translated into mutual information between the data and model reconstructions, $I(X; \hat{X})$.  While this should already encourage encoders to match a given marginal, the MMD penalty appears useful in enforcing this condition throughout training for VAEs.}

 \deleted{ A learned approximation such as MAF or VampPrior could also help ensure a tight rate bound while adding flexibility to learn a complex marginal space.  However, these effects can be difficult to distinguish. Gaussian encoders appear unable to leverage the complexity of the MAF transformation, with prior consistency proving particularly useful on Omniglot.  VampPrior is useful for both encoders and \textit{IAF+Vamp} provides competitive performance across datasets, although we cannot easily quantify the gap in the marginal approximation. Alternatively, Echo provides both an exact rate and an adaptive prior by directly linking the choice of encoder and marginal.}  



\subsection{Rate Distortion Curves}\label{sec:rdresults}

\begin{figure*}[t]
\centering
\begin{minipage}{.5\textwidth}
  \centering 
    \includegraphics[width=\textwidth]{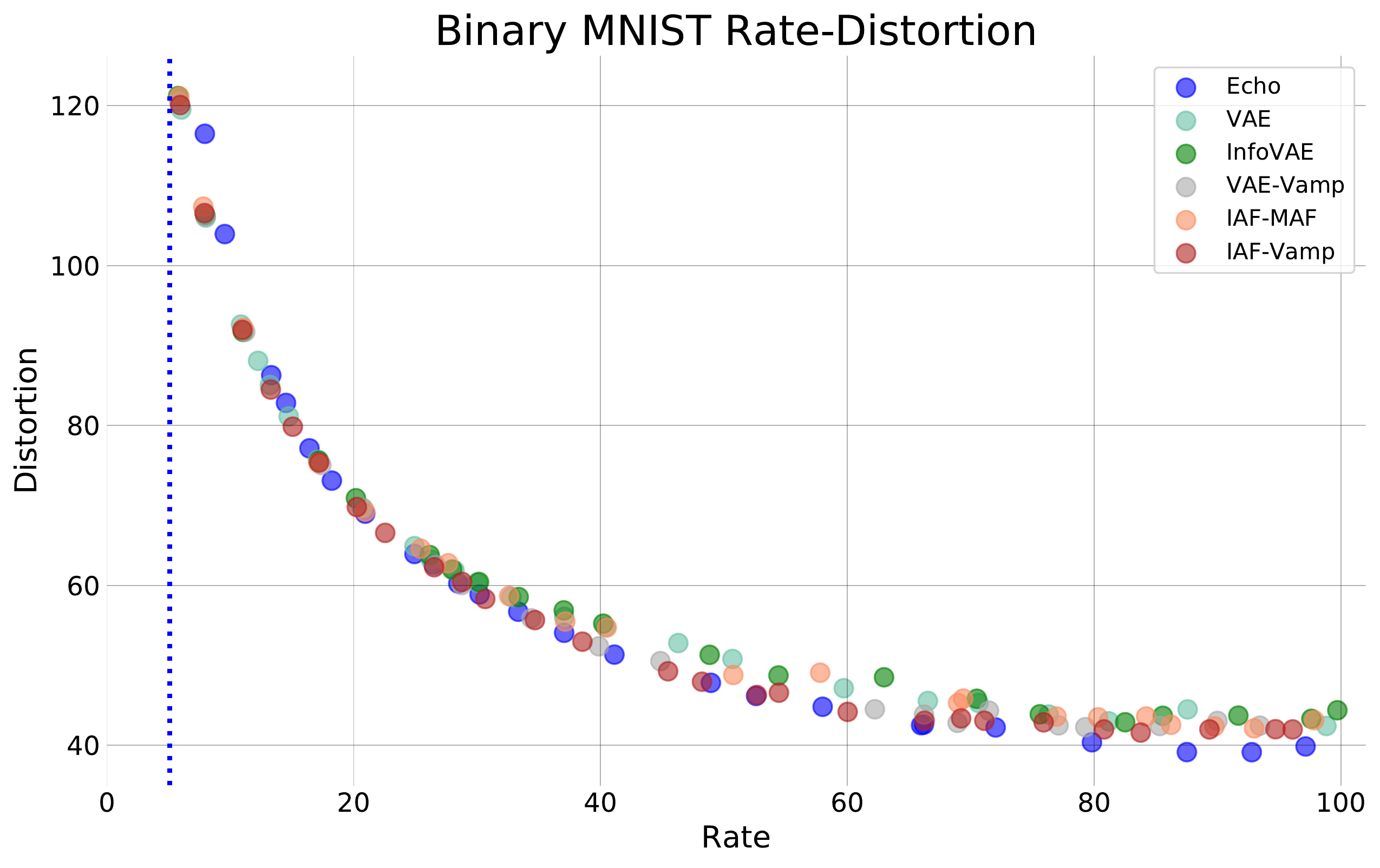}
    \includegraphics[trim=25 25 15 25, clip, width=\textwidth]{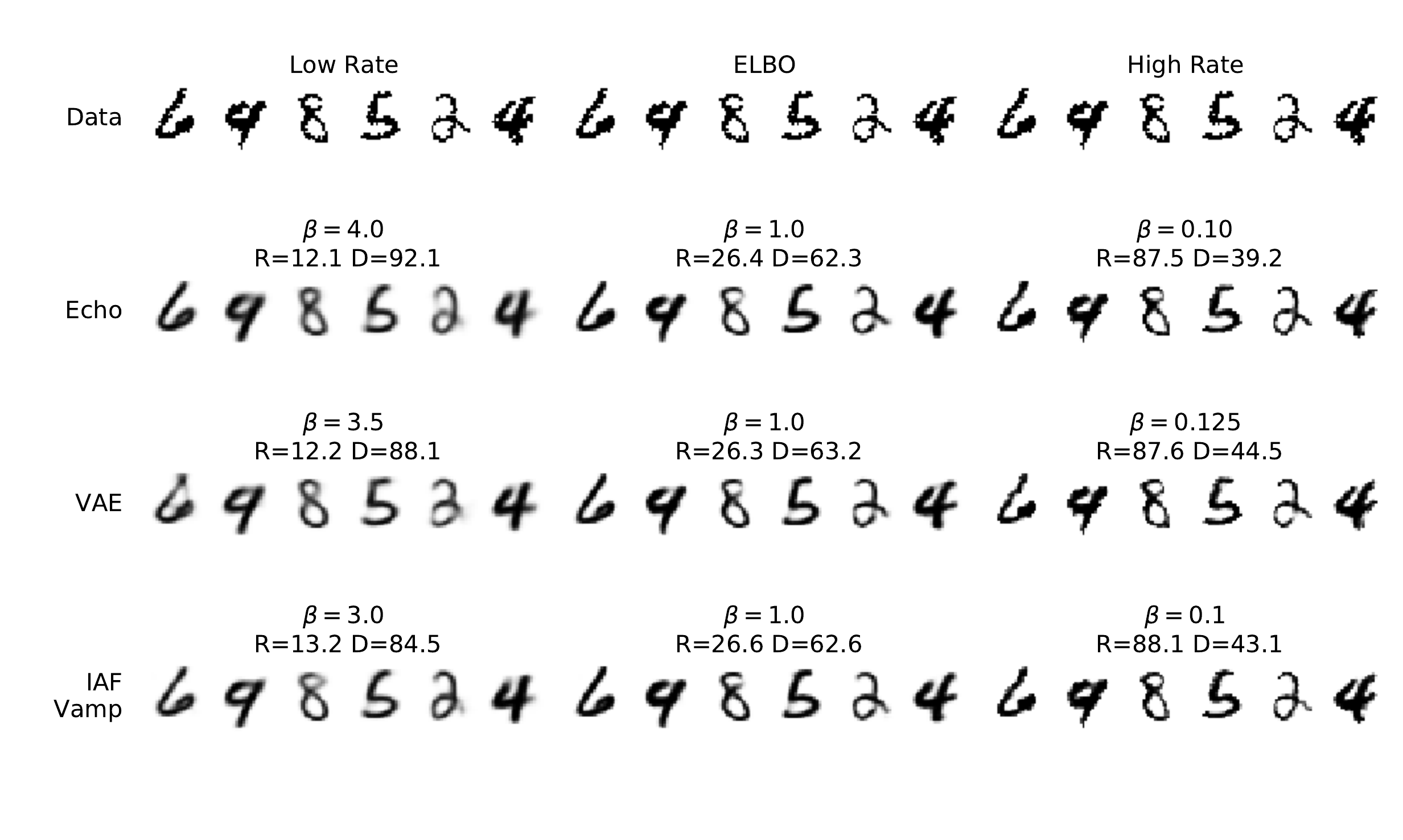}

\captionof{figure}{Binary MNIST R-D and Visualization} \label{fig:bm_rd}
\end{minipage}%
\begin{minipage}{.5\textwidth}
  \includegraphics[width=\textwidth]{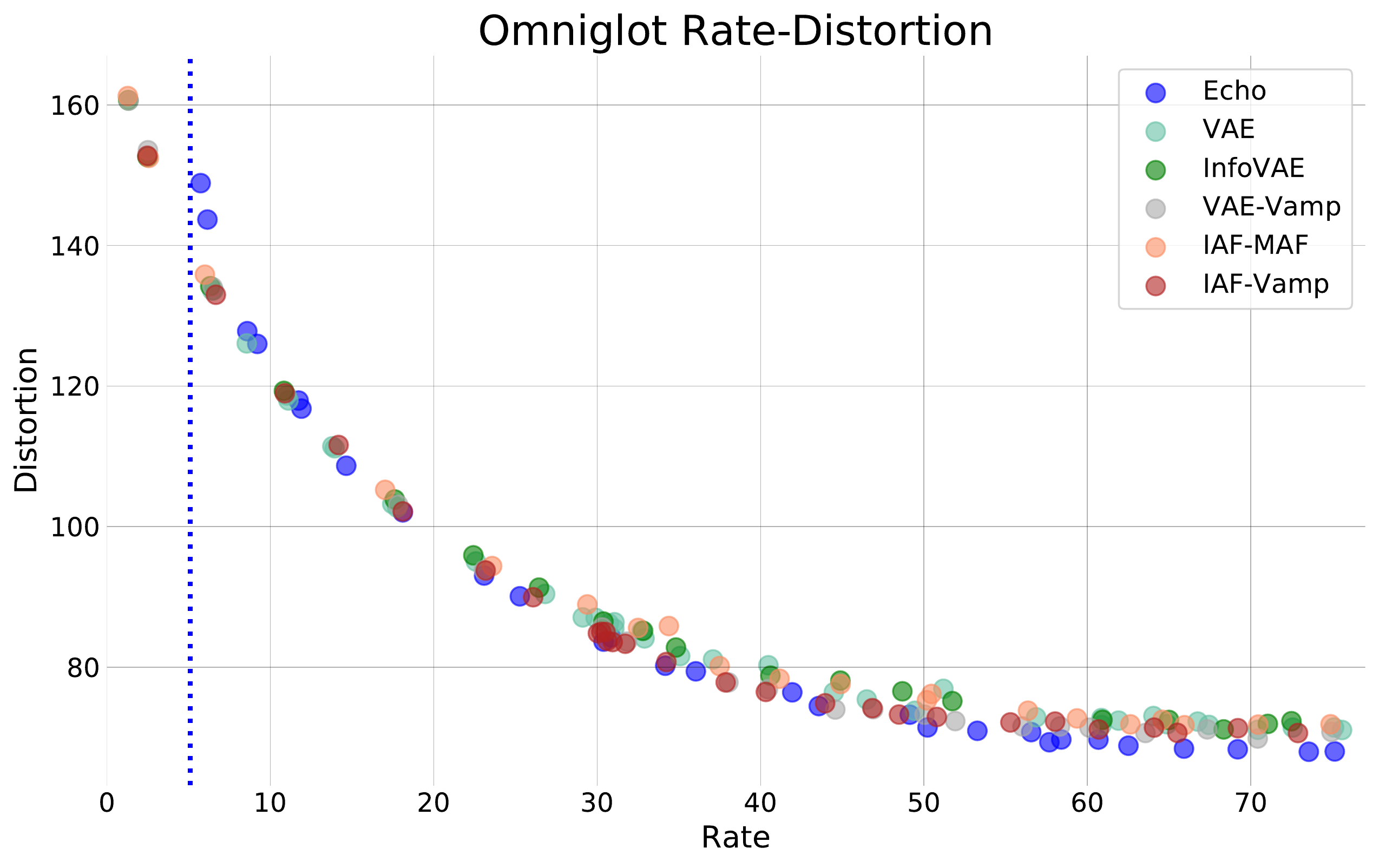}
  \vspace{-.005cm}
  \includegraphics[trim=10 25 15 15, clip, width=\textwidth]{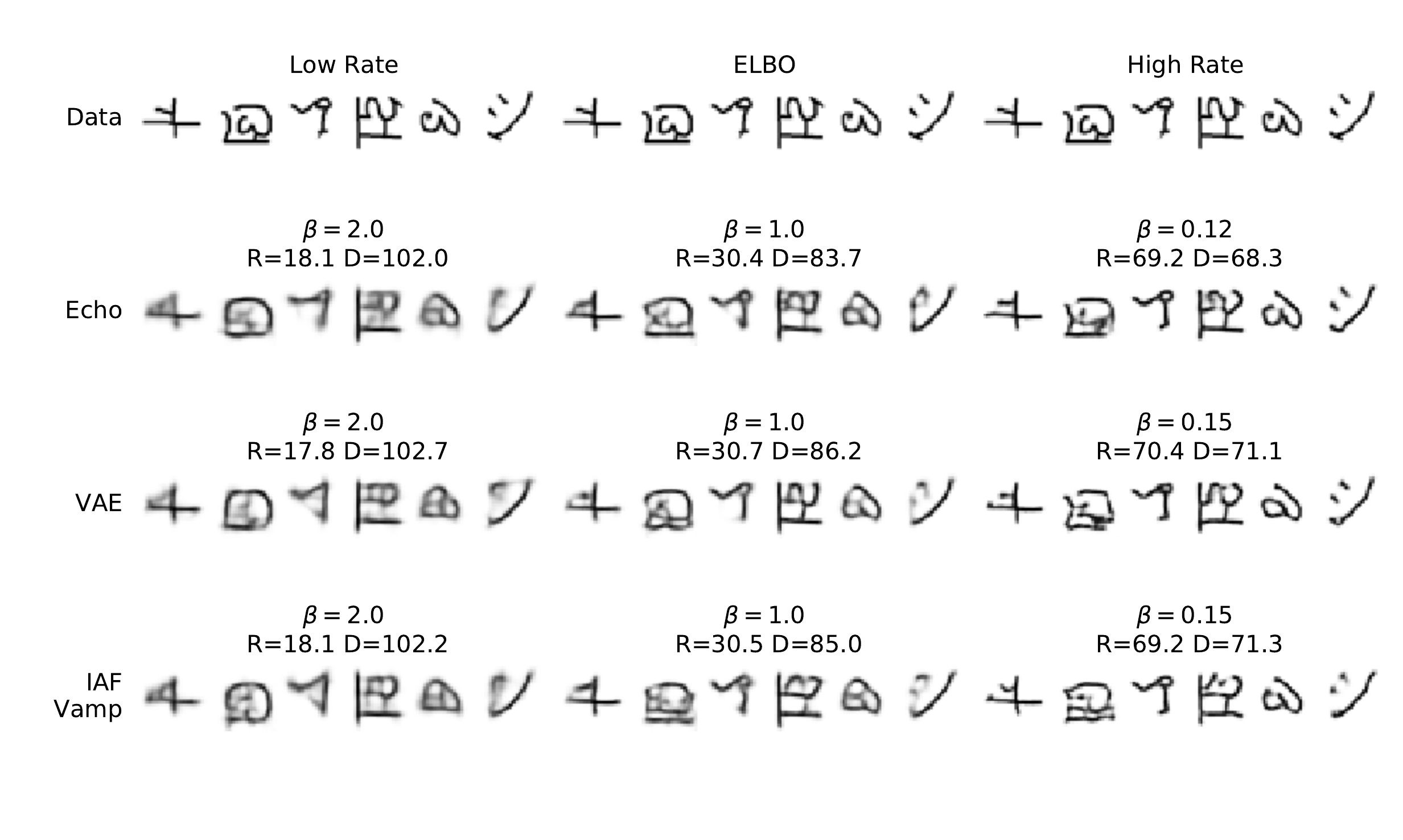}
    \captionof{figure}{Omniglot R-D and Visualization}
   \label{fig:omni_rd}
\end{minipage}
\end{figure*}

Moving beyond the special case of $\beta = 1$, rate-distortion theory provides the practitioner with an entire space of compression-relevance tradeoffs corresponding to constraints on the rate.  We plot R-D curves for Binary MNIST in Fig.~\ref{fig:bm_rd},  Omniglot in Fig.~\ref{fig:omni_rd}, and Fashion MNIST in App.~\ref{app:fm_rd}.  We also show model reconstructions at several points along the curve, with the output averaged over 10 encoding samples to observe how stochasticity in the latent space is translated through the decoder.  These visualizations are organized to compare models with similar rates, which we emphasize may occur at different values of $\beta$ for different methods depending on the shape of their respective curves.   

\replaced{The Echo rate-distortion curve indeed exhibits several notable differences with comparison methods.  We first note that Echo performance begins to drop off as we approach the lower limit on achievable rate, which is shown with a dashed vertical line and ensures that the rate calculation accurately reflects the noise for a finite number of samples (see App.\ref{app:implementation}).  In this regime, the sigmoids parameterizing $s_j(\vx)$ are saturated for much of training, and unused dimensions still count against the objective since we cannot achieve zero rate.   We reiterate that this low rate limit may be adjusted by considering more terms in the infinite sum or decreasing the number of latent factors. }{\\ \\ Echo continues to outperform comparison methods across most of the rate-distortion curve.  However, we note that performance begins to drop off as we approach the lower limit on achievable rate, shown with a dashed vertical line in each plot and described in detail in App.\ref{app:implementation}.  This bound arises from ensuring that the rate calculation accurately reflects the noise for a finite number of samples, and is described in detail in App.\ref{app:implementation}.  In this regime, the sigmoids parameterizing $s_j(\vx)$ are saturated for much of training, with unused dimensions still counting against the rate.   We reiterate that this low rate limit may be adjusted by considering more terms in the infinite sum or decreasing the number of latent factors.}

At low rates, our models maintain only high level features of the input image, and the blurred average reconstructions reflect that different samples can lead to semantically different generations.  \deleted{For example, Echo produces several ways of rounding the leftmost digits on Binary MNIST.}  On both datasets, Echo gives qualitatively different output variation than Gaussian noise at low rate and similar distortion.  Intermediate-rate models still reflect some of this sample diversity, particularly on the more difficult Omniglot dataset.

For very high capacity models, we observe that Echo slightly extends its gains on both datasets, with three to five nats lower distortion than comparison methods at the same rates.   Intuitively, a more complex encoding marginal may be harder to match to a (learned) prior, loosening the upper bound on mutual information.  The Echo approach can be particularly useful in this regime, as it avoids explicitly constructing the marginal while still providing exact rate regularization.

\begin{figure}[t] 
\begin{minipage}{.7\linewidth}
\centering
\captionof{table}{Disentanglement Scores} 
\label{tab:disent2}
\vspace{.1cm}
\begin{tabular}{lcccccccccccccc}
 & \multicolumn{6}{c}{\small Independent Ground Truth} & \multicolumn{6}{c}{\small Dependent Ground Truth} \\ \midrule
  & \multicolumn{2}{c}{\small Factor} & & \multicolumn{2}{c}{\small MIG} & & & \multicolumn{2}{c}{\small Factor } & &  \multicolumn{2}{c}{MIG}\\ \cmidrule(r){2-3} \cmidrule(r){5-6} \cmidrule(r){9-10} \cmidrule(r){12-13}
  & \small Echo & \small VAE & & \small Echo & \small VAE & & & \small Echo  & \small VAE & & \small Echo & \small VAE  \\ \midrule
\small $\beta = 1$ & \textbf{0.83} & 0.65  & &  0.16 & 0.07 & & & \textbf{0.70 }& 0.60  & & 0.11 & 0.08 \\
\small $\beta = 4$ & 0.78 & 0.65  & & 0.18 & 0.10 & & & 0.67 & 0.60 & & 0.11  & 0.07 \\ 
\small $\beta = 8$ & 0.75 & 0.69 & &  0.18  & 0.13 & & &  0.56 & 0.56 & & 0.06 & 0.06 \\\midrule
\small $\gamma=0$  & 0.83 & 0.65 & &  0.16  & 0.07 & & &  0.70 & 0.60 & &  0.10 & 0.08 \\ 
\small $\gamma=20$ &  0.78 & 0.72 & & \textbf{0.30} & 0.17 & & & 0.65 &  0.60 & &  \textbf{0.16} & 0.07 \\
\small $\gamma=50$&  0.79 & 0.73  & &  \textbf{0.30}  & 0.18 & & & 0.58 & 0.53 & & \textbf{0.16} & 0.07 \\ 
\small $\gamma=100$&  0.77 & 0.70 & & 0.29 & 0.18 & & & 0.49 &  0.53 & & 0.09   & 0.08 \\ 
 \bottomrule
\end{tabular}%
\end{minipage}%
\begin{minipage}{.3\linewidth} 
    \vspace{+1.4\baselineskip}
    \captionof{figure}{\protect\centering Echo \small $\beta = 0, \gamma = 1$} \label{fig:disent_plot} 
    \includegraphics[trim=10 0 20 40, clip, width=\textwidth,height=5.25cm]{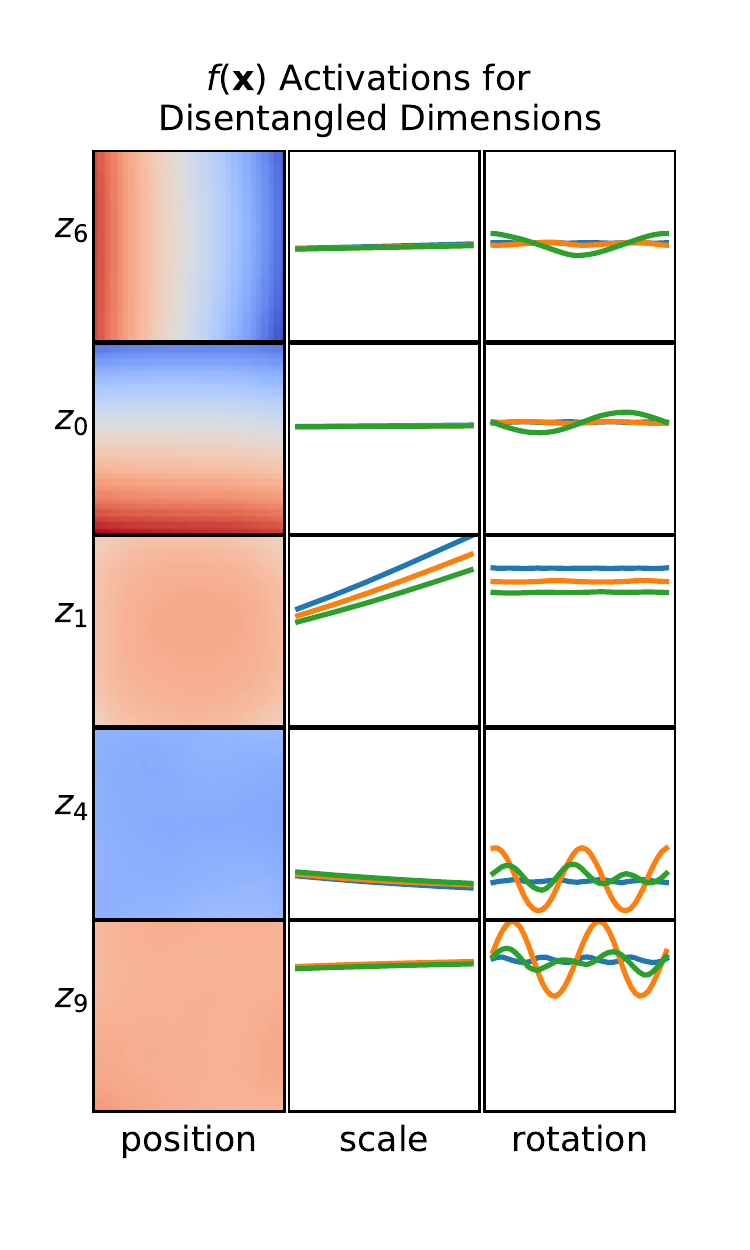}
  \end{minipage}%
\end{figure}

\subsection{Disentangled Representations} \label{sec:disentanglement}
Significant recent attention has been devoted to learning \textit{disentangled} representations of data, which reflect the true generative factors of variation in the data \cite{chen2018isolating, mathieu2019disentangling} and may be useful for downstream tasks \cite{locatello2018challenging, van2019disentangled}.   While prevailing definitions and metrics for disentanglement have recently been challenged \cite{locatello2018challenging}, existing methods often rely on the inductive bias of independent ground truth factors, either via total correlation (TC) regularization \cite{chen2018isolating,kim2018disentangling}, or by using higher  $\beta$  to more strongly penalize the KL divergence to an independent prior \cite{burgess2018understanding, higgins2017}.  Since Echo does not assume a factorized encoder or marginal, we \replaced{investigate whether it can better preserve disentanglement when the ground truth factors are not independent. }{ hypothesize that it may be more flexible in preserving disentanglement when the ground truth factors are not independent.}  

To evaluate the quality of Echo noise representations, we compare against VAE models with diagonal Gaussian noise and priors, and consider the effects of increasing $\beta$ or adding independence regularization with parameter $\gamma$  \cite{chen2018isolating, kim2018disentangling}:
\begin{align*}
    \mathcal{L} = \mathbb{E}_{q} \log \dec - \beta \iq - \gamma \, TC(Z)
\end{align*}
TC regularization is implemented as in \cite{kim2018disentangling}, where a discriminator is trained to distinguish samples from $q(\vz)$ and $\prod q(z_j)$.  We keep $\beta = 1$ when modifying $\gamma$.  Note that enforcing marginal independence will also limit the dependence in the noise learned by Echo, since $TC(Z|X)$ and $TC(Z)$ are linked as described in Sec. \ref{sec:echo_properties}.

We calculate disentanglement scores on the dSprites dataset \cite{dsprites17}, where the ground truth factors of shape, scale, x-y position, and rotation are known and sampled independently across the dataset.  To induce dependence in the ground truth factors, we downsample the dataset by partitioning each factor into 4 bins and randomly excluding pairwise combinations of bins with probability 0.15.  This leads to a dataset of 15\% of the original size, with a total correlation of 1.49 nats in the generative factors.  We use both the implementation and experimental setup of \citet{locatello2018challenging} and average scores over ten runs of each method.

Table \ref{tab:disent2} reports FactorVAE \cite{kim2018disentangling} and Mutual Information Gap \cite{chen2018isolating} scores for both independent and dependent ground truth factors.  We find that Echo provides superior disentanglement scores to VAEs across the board, although the relative improvement does not increase in the case of dependent latent factors.   On the full dataset, independence regularization improves the MIG score for Echo and both scores for VAE, but may guide both models toward more entangled representations when this inductive bias does not match the ground truth.  Finally, we note that increasing $\beta$ need not improve disentanglement for Echo noise, since we have relaxed assumptions of independence in both the encoder and marginal.  Higher $\beta$ actually appear to hurt disentanglement scores on the dependent dataset for both methods.

In Figure \ref{fig:disent_plot}, we visualize an Echo model that has successfully learned to disentangle position and scale, but not rotation, on the full dSprites dataset.  Each row represents a single latent dimension, and each column shows mean $f(\vx)$ values as a function of the respective ground truth factors.  Note that the first column shows a heatmap in the x-y plane, while the orange, blue, and green lines indicate ellipse, square, and heart, respectively (see \cite{chen2018isolating}).  In general, we observed that Echo models achieved their highest MIG scores on position, scale, and shape, with rotation often entangled across two or more dimensions.

\section{Conclusion}\label{sec:conclusion}

VAEs can be interpreted as performing a rate-distortion optimization, but may be handicapped by their weak compression mechanism, independent Gaussian marginal assumptions, and upper bound on rate.  We introduced a new type of channel, Echo noise, that provides a more flexible, data-driven approach to constructing noise and admits an exact expression for mutual information.  Our results demonstrate that using Echo noise in autoencoders can lead to better bounds on log-likelihood, favorable trade-offs between compression and reconstruction, and \added{more disentangled representations.}

The Echo channel can be substituted for Gaussian noise \replaced{in most scenarios where VAEs are used}{}, with similar runtime and the same number of parameters.  Echo should also translate to other rate-distortion problems via the choice of distortion measure, including supervised learning with the traditional Information Bottleneck method \cite{alemi2016deep, tishby2000information} and \added{invariant representation learning as in \cite{ moyer2018invariant}}. Exploring further settings where  mutual information provides meaningful regularization for neural network representations remains an exciting avenue for future work.

\bibliographystyle{plainnat}
\bibliography{sections/gversteeg_bibdesk}
\clearpage
\appendix
\newpage



\section{Rate-Distortion Theory}\label{app:rd}
Given a source $X \sim \qd$ and a distortion function $d : \mathcal X \times \mathcal Z \mapsto \mathbb{R}^{+}$ over samples and their codes $Z$, the rate-distortion function is defined as an optimization over conditional distributions $q(\vz|\vx)$:
\begin{align}
R(D) & = \min_{q(\vz|\vx)} I_{q}(X;Z) \quad \text{ subj  }\mathbb{E}_{\qd q(\vz|\vx)} d(\vx, \vz) \leq D \label{eq:rd}
\end{align}
It is common to optimize an unconstrained problem by introducing a Lagrange multiplier $\beta^{-1}$ which, at optimality, reflects the tradeoff between compression and fidelity as the slope of the rate-distortion function at $D$, i.e. $\beta^{-1} = -\frac{\partial R}{\partial D}$: \footnote{Note, we have constrained the distortion here, instead of the rate as in the main text.  We write the Lagrange multiplier as $\beta^{-1}$ to maintain a correspondence between the parameterizations of each problem.}
\be
\mathcal L = \max_{\beta} \min_{q(\vz|\vx)} I_{q}(X;Z) + \beta^{-1} \, \big( \mathbb{E}_{\qd q(\vz|\vx)} d(\vx, \vz) - D \big) \nonumber 
\ee

Eq. \ref{eq:problemdef} suggests the cross entropy reconstruction loss as a distortion measure, so that $d(\vx, \vz) = - \log \dec$.  We can then observe the equivalence between the rate-distortion optimization and our problem definition, as only the tradeoff between rate and distortion affects the characterization of solutions.  

It is also interesting to note the self-consistent equations which solve the variational problem above (see, e.g.  \citet{tishby2000information})
\be
q(\vz|\vx) &= \frac{q(\vz)}{Z(\vx, \beta)} \exp \big(-\beta^{-1} \, d(\vx, \vz) \big) \label{eq:rdsolution} \\
q(\vz) &= \int q(\vz|\vx) q(\vx) d{\vx} \nonumber
\ee
Notice that, regardless of the choice of distortion measure, our Echo noise channel enforces the second equation \textit{throughout} optimization by using the  encoding marginal as the `optimal prior.'  For our choice of distortion, the solution simplifies as:  %
\be
q(\vz|\vx) &= \frac{q(\vz)p(\vx|\vz)^{1/\beta}}{Z(\vx, \beta)} \label{eq:rd_unsupervised}
\ee
This provides an interesting comparison with the generative modeling approach.  While the Evidence Lower Bound objective can be interpreted as performing posterior inference with prior $p(\vz)$ in the numerator, we see that the information theoretic perspective prescribes using the exact encoding marginal $q(\vz)$.  Indeed, our version of the ELBO bounds in Eq. \ref{eq:tight_elbo} bounds the likelihood under the generative model $p(\vx) = \int \qz \dec d{\vx}$.  The gap in this bound then becomes $\kl [\enc || \frac{q(\vz)p(\vx|\vz)}{Z(\vx)} ]$, encouraging the encoder to match the rate-distortion solution for $\beta = 1$.

\section{Implementation of Echo Noise Sampling}\label{app:implementation}
Numerically, Gaussian noise cannot be sampled exactly and is instead approximated to within machine precision.  We discuss several unique implementation choices that allow us to generate similarly precise Echo noise samples.  In particular, we must ensure that the infinite sum defining the noise in Eq.\ref{eq:noise} converges and is accurately approximated using a finite number of terms. 
%

\textbf{Activation Functions: }  We parameterize the encoding functions $f(\vx)$ and $S(\vx)$ using a neural network and can choose our activation functions to satisfy the convergence conditions of Lemma \ref{condition}.  We let the final layer of $f$ use an element-wise $\tanh{(\cdot / 16)}$ to guarantee that the magnitude is bounded: $\forall \vx , \, |f(\vx)| \leq 1$.  We found it useful to expand the linear range of the $\tanh$ function for training stability, although differences were relatively minor and may vary by application. One could also consider clipping the range of a linear activation to enforce a desired magnitude $|f(\vx)| \leq M$.

For the experiments in this paper, $S(\vx)$ is diagonal, with functions $s_j(\vx)$ on the diagonal.  We implement each $s_j(\vx)$ using a sigmoid activation, making the spectral radius $\rho(S(\vx)) = \max_j |s_j(\vx)| \leq 1$.  However, this is not quite enough to ensure convergence, as $\forall \vx, s_j(\vx) = 1$ would lead to an infinite amount of noise.  We thus introduce a clipping factor on $s_j(\vx)$ to further limit the spectral radius and ensure accurate sampling in this high noise, low rate regime.

\textbf{Sampling Precision: } When can our infinite sum be truncated without sacrificing numerical precision?  We consider the sum of the remainder terms after truncating at $\ell = \dmax$ using geometric series identities.  For $|f(\vx)| \leq M$ and $\rho(S(\vx)) \leq \rad$, we know that the sum of the infinite series will be less than $\frac{M}{1-\rad}$ . The first $\dmax$ terms will have a sum given by $M \big( \frac{1-\rad^{\dmax}}{1-\rad}\big)$, so the remainder will be less than $M \big( \frac{\rad^{\dmax}}{1-\rad}\big)$.  For a given choice of $\dmax$, we can numerically solve for $\rad$ such that the sum of truncated terms falls within machine precision $M \big( \frac{\rad^{\dmax}}{1-\rad}\big) \leq 2^{-23}$.  For example, with $M = 1$ and $\dmax = 99$, we obtain $\rad = 0.8359.$  We therefore scale our element-wise sigmoid to $s_j(\vx) = \rad \sigma(\cdot)$ for calculating both the noise and the rate.    

\textbf{Low Rate Limit: } \label{sec:clip}
This clipping factor limits the magnitude of noise we can add in practice, and thus defines a lower limit on the achievable rate in an Echo model.  For diagonal $S(\vx)$, the mutual information can be bounded in terms of $\rad$, so that $I(X;Z) = -\sum_{j=1}^{d_z} \mathbb E_q \log |s_j(\vx)| \geq -d_z \log \rad$.  Note that $\rad$ is increasing in $\dmax$, since the first term in the remainder decreases exponentially with the number of terms.  Each included term can then have higher magnitude, leading to lower achievable rates.  Thus, this limit can be tuned to achieve strict compression by increasing $\dmax$ or simply using fewer latent factors $d_z$. 

\textbf{Batch Optimization: } Another consideration in choosing $\dmax$ is that we train using mini-batches of size $B$ for stochastic gradient descent.  For a given training example, we can use the other iid samples in a batch to construct Echo noise, thereby avoiding additional forward passes to evaluate $f$ and $S$.  There is also a choice of whether to sample with or without replacement, although these will be equivalent in the large batch limit. In experiments we saw little difference between these strategies, and proceed to sample without replacement to mirror the treatment of training examples. We let $\dmax = B-1$ to set the rate limit as low as possible for this sampling scheme.

\section{Total Correlation for Echo Noise} \label{app:tcrelation}\label{app:tc}
To briefly demonstrate that Echo noise is dependent across latent dimensions, we can estimate the total correlation of noise samples in Table \ref{tab:tc} using the second-order covariance approximation $TC(\bfeps) = -\log |\Sigma_{\mathrm{diag}_{\epsilon}^{-1}} \Sigma_{\bfeps} |$.  This is clearly zero for diagonal Gaussian noise, and provides a sufficient condition to show that the Echo noise is not independent.

\begin{table}[h]
\centering
\caption{TC by Dataset}
\label{tab:tc}
\begin{tabular}{lccc}
\toprule 
    & Binary MNIST &  Omniglot & Fashion MNIST  \\ \midrule
 $TC(\bfeps)$ & 7.3 & 18.8 & 30.2 \\
\bottomrule
\end{tabular}
\end{table}

For the Echo models considered in this work, we can also derive an interesting equivalence between the conditional and overall total correlation.  Observe that the expression for mutual information in Eq.~\ref{eq:mi} decomposes for diagonal $S(\vx)$:
\begin{align*}
\iq &= - \mathbb{E}_{\qd} \, \log  |\det S(X)| \\
& = - \mathbb{E}_{\qd} \sum \limits_{j = 1}^{d_z}  \log s_j(X)
\end{align*}
This additivity across dimensions implies that $\iq = \sum_{j = 1}^{d_z} I_{q}(X;Z_j)$.  Before proceeding, we first recall the definitions of total correlation and conditional total correlation \cite{watanabe}, which measure the divergence from independence of the marginal and conditional, respectively:
\benn
TC(Z) &= \kl [ \qz || \prod \limits_{j = 1}^{d_z} q_{\phi}(z_j) ] \\
TC(Z|X) &= \kl [ \enc || \prod \limits_{j = 1}^{d_z} q_{\phi}(z_j | \vx) ]
\eenn
Now consider the quantity $\kl [ \enc || \prod \limits_{j = 1}^{d_z} q_{\phi}(z_j)]. $  We can decompose this in two different ways, first by projecting onto the joint marginal:
\benn
\kl [ \enc || \prod_{j = 1}^{d_z} q_{\phi}(z_j)] & = \eq \log \frac{ \enc }{ \prod_{j = 1}^{d_z} q_{\phi}(z_j)} \\
& = \eq \log \frac{ \enc }{ \prod_{j = 1}^{d_z} q_{\phi}(z_j)} \frac{\qz }{\qz } \\
& = \iq + TC(Z)
\eenn
We can also decompose using the factorized conditional:
\benn
\kl [ \enc || \prod_{j = 1}^{d_z} q_{\phi}(z_j)]  &= \eq \log \frac{ \enc }{ \prod_{j = 1}^{d_z} q_{\phi}(z_j)} \\
& = \eq \log \frac{ \enc }{ \prod_{j = 1}^{d_z} q_{\phi}(z_j)} \frac{\prod_{j = 1}^{d_z} q_{\phi}(z_j | \vx) }{\prod_{j = 1}^{d_z} q_{\phi}(z_j | \vx) } \\
& = \sum \limits_{j = 1}^{d_z} I_{q}(X;Z_j) + TC(Z|X)
\eenn
The equality of $I(X;Z)$ and $\sum \limits_{j = 1}^{d_z} I_{q}(X;Z_j)$ implies equality for $TC(Z)$ and $TC(Z|X)$.
\benn
 \iq + TC(Z) &= \sum \limits_{j = 1}^{d_z} I_{q}(X;Z_j) + TC(Z|X) \\
 \implies TC(Z) &= TC(Z|X)
\eenn
The effects of this relationship have not been widely studied, as $TC(Z|X)=0$ for traditional VAE models.  On the other hand,  $TC(Z)$ is usually non-zero and has been minimized as a proxy for `disentanglement' \cite{kim2018disentangling, chen2018isolating}.  We evaluate similar regularization for Echo in Sec. \ref{sec:disentanglement}.  

We have shown that parallel Echo channels are perfectly additive in that $\sum_j I_{q}(X;Z_j) - \iq = 0$.  However, general channels could be sub- or super-additive, so that $TC(Z) < TC(Z|X)$, $TC(Z) = TC(Z|X)$, or $TC(Z) > TC(Z|X)$ (e.g. Sec. 4.2 of \cite{griffith}).  Extending Echo to non-diagonal $S(\vx)$ could allow us to explore the various relationships between $TC(Z)$ and $TC(Z|X)$ and more precisely characterize those which are useful for representation learning.


\section{Additional Results} \label{app:additional}
\subsection{Fashion MNIST Rate-Distortion} \label{app:fm_rd}

We show a full rate-distortion curve for Fashion MNIST in Fig.\ref{fig:fm_rd}, along with reconstructions at various rates.  Echo performance nearly matches that of comparison methods except at low rates.

\begin{minipage}[b]{\columnwidth}
  \centering
    \includegraphics[width=.5\columnwidth]{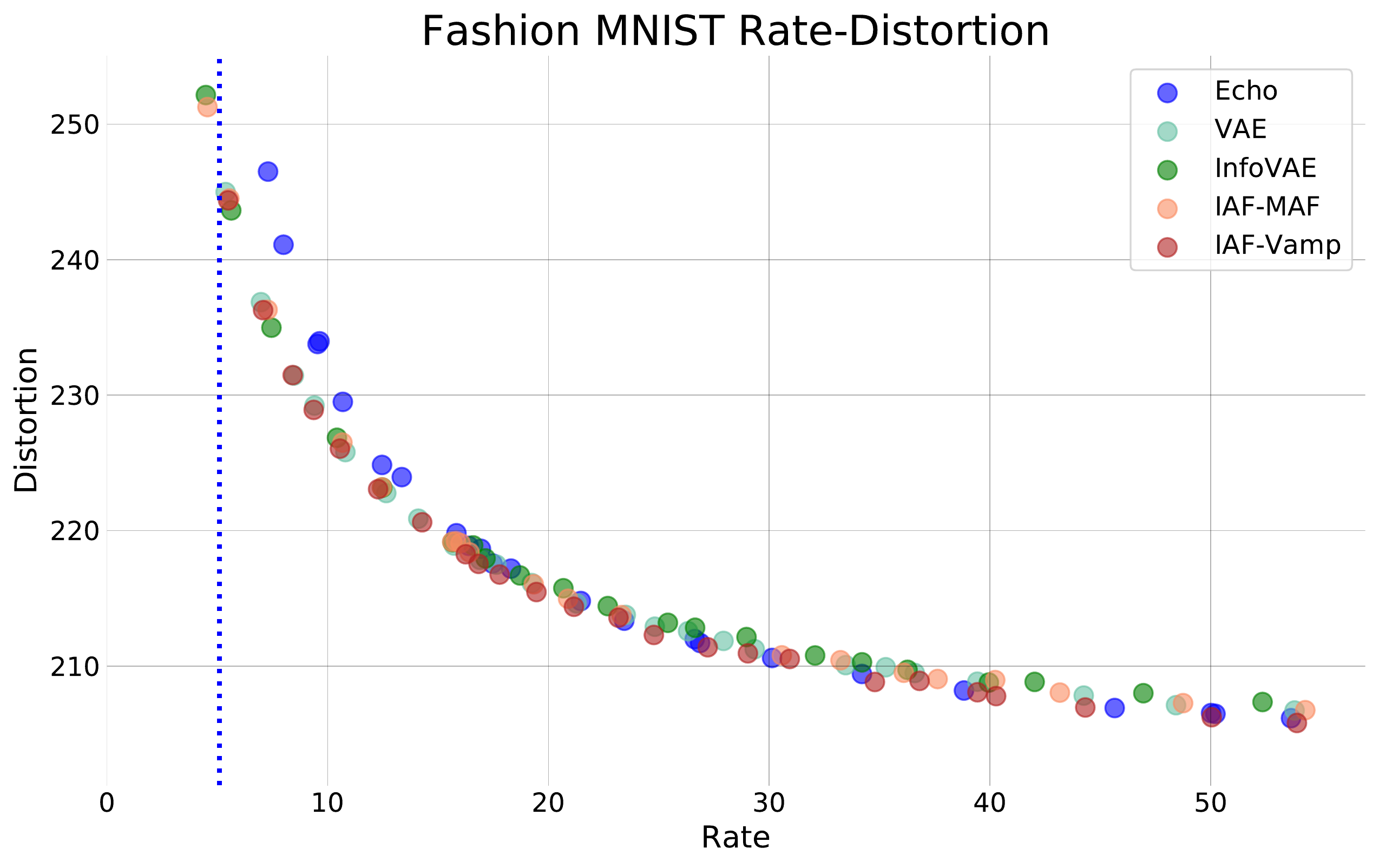}%
    \includegraphics[width=.5\columnwidth]{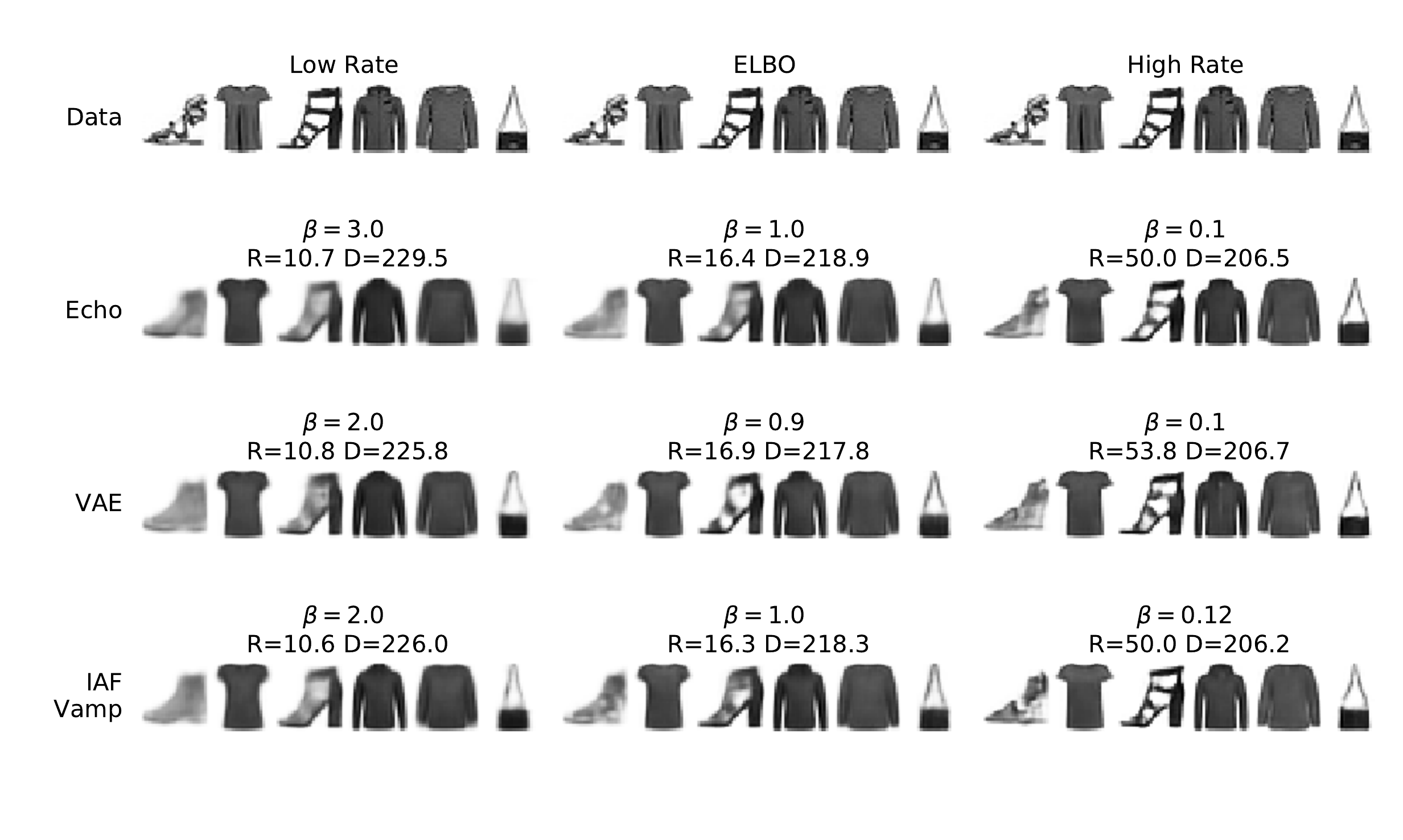}
\captionof{figure}{FMNIST Rate-Distortion and Visualization} \label{fig:fm_rd}
\end{minipage}

\subsection{Marginal Activations}\label{app:marginals}
We visualize dimension-wise marginal activations for Echo on Binary MNIST and Omniglot in Fig.\ref{fig:marginal_act}.  We show $q(z_j)$ for thirteen dimensions in each method, including nine with highest rates, three with low rates, and one with minimal rate.  For each, we combine activations from 2000 encoder samples on each test example and fit a KDE estimator with RBF bandwidth chosen according to the Scott criterion.  

As discussed in Sec.~\ref{sec:echo}, Echo avoids assumptions that the marginals are independent and Gaussian as in VAEs.  However, we observe the individual Echo marginals $q(z_j)$ to be approximately Gaussian, with the Anderson-Darling test failing to reject the null hypothesis of Gaussianity for any dimension.  Nevertheless, the joint marginal $q(\vz)$ may still be dependent (see App.~\ref{app:tc}).

Individual dimensions are also are free to learn different means and variances without incurring a penalty in the objective, with factors generally keeping more mutual information with the data having less variance in the marginals.  The highest mean dimension in the Omniglot plot corresponds to an `unused' dimension that saturates the lower limit on achievable rate.

\begin{minipage}[t]{\textwidth}
\includegraphics[trim=40 0 0 0, clip, width=.5\textwidth]{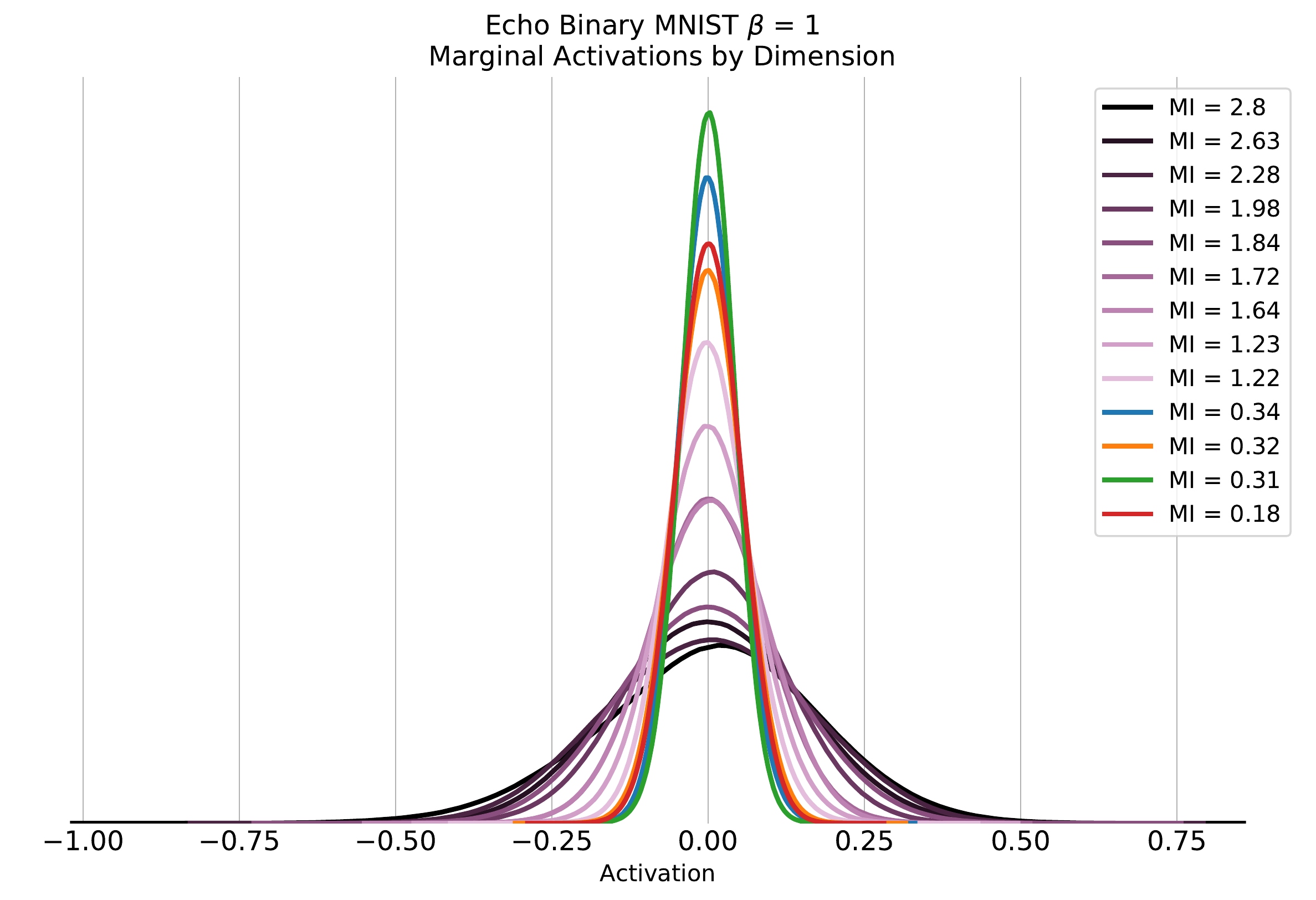}%
\includegraphics[trim=40 0 0 0, clip, width=.5\textwidth]{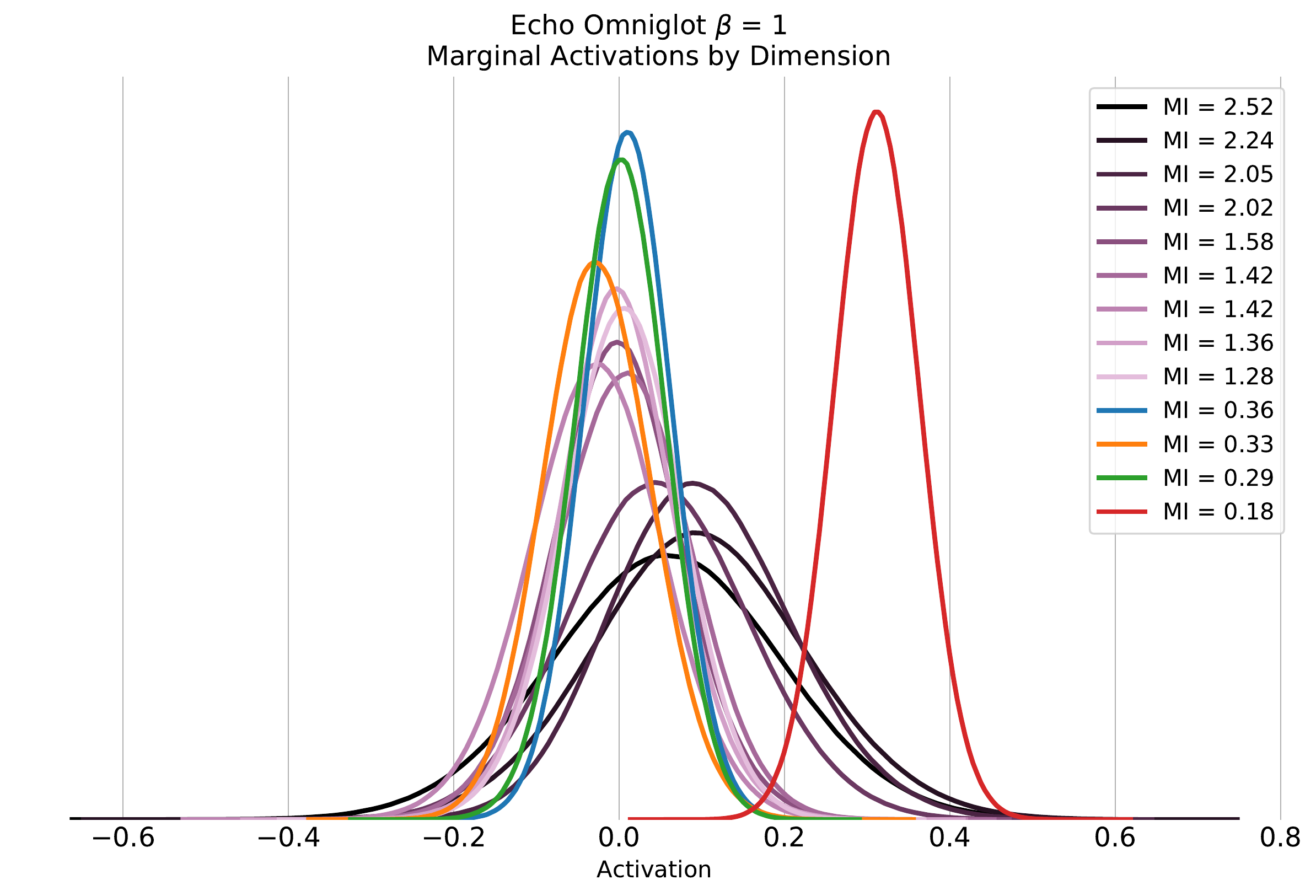}
\captionof{figure}{Marginal Activations by Dimension} \label{fig:marginal_act}
\end{minipage}
\subsection{Echo $f(\vx)$ vs. $S(\vx)$}
\label{app:fxsx}

We can analyse the Echo mutual information at each data point by noting that the expression in Eq.~\ref{eq:mi} involves an expectation over $\vx$.  Since $H(Z)$ and $H(\mathcal{E})$ do not depend on $X$ in the proof of Thm.~\ref{thm:mi}, we can evaluate $- \sum_j \log s_j(\vx)$  as a pointwise mutual information.  We compare this quantity with the L2-norm of $f(\vx)$ as a proxy for signal to noise ratio.  Test examples are sorted by conditional likelihood $\dec$ on the x-axis, and we see that Echo indeed has higher mutual information on examples where the generative model likelihood is high.  Further analysis of these pointwise informations remains for future work.

\begin{minipage}[ht]{\columnwidth}
  \centering
  \includegraphics[width=0.5\columnwidth]{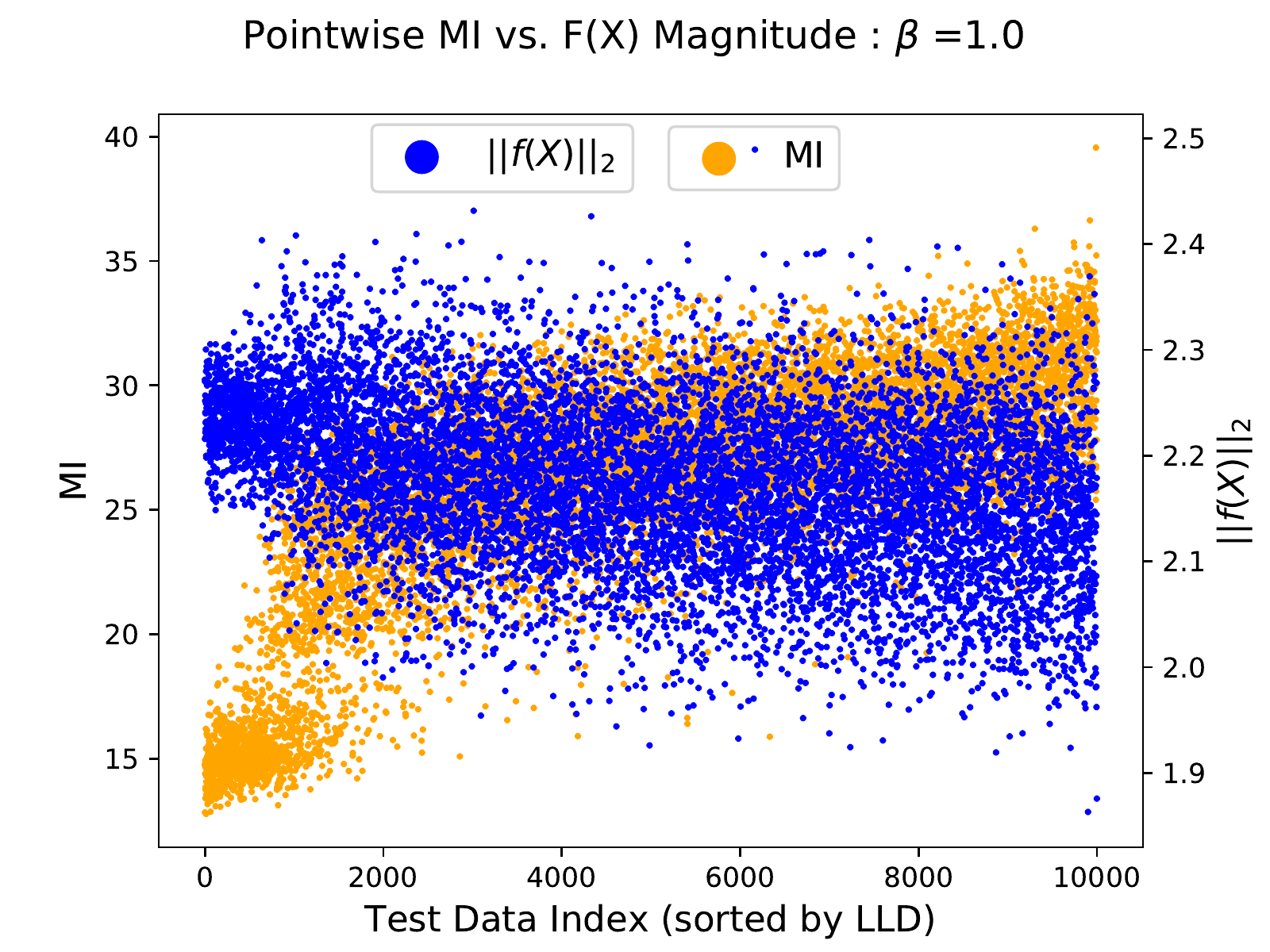}%
    \includegraphics[width=0.5\columnwidth]{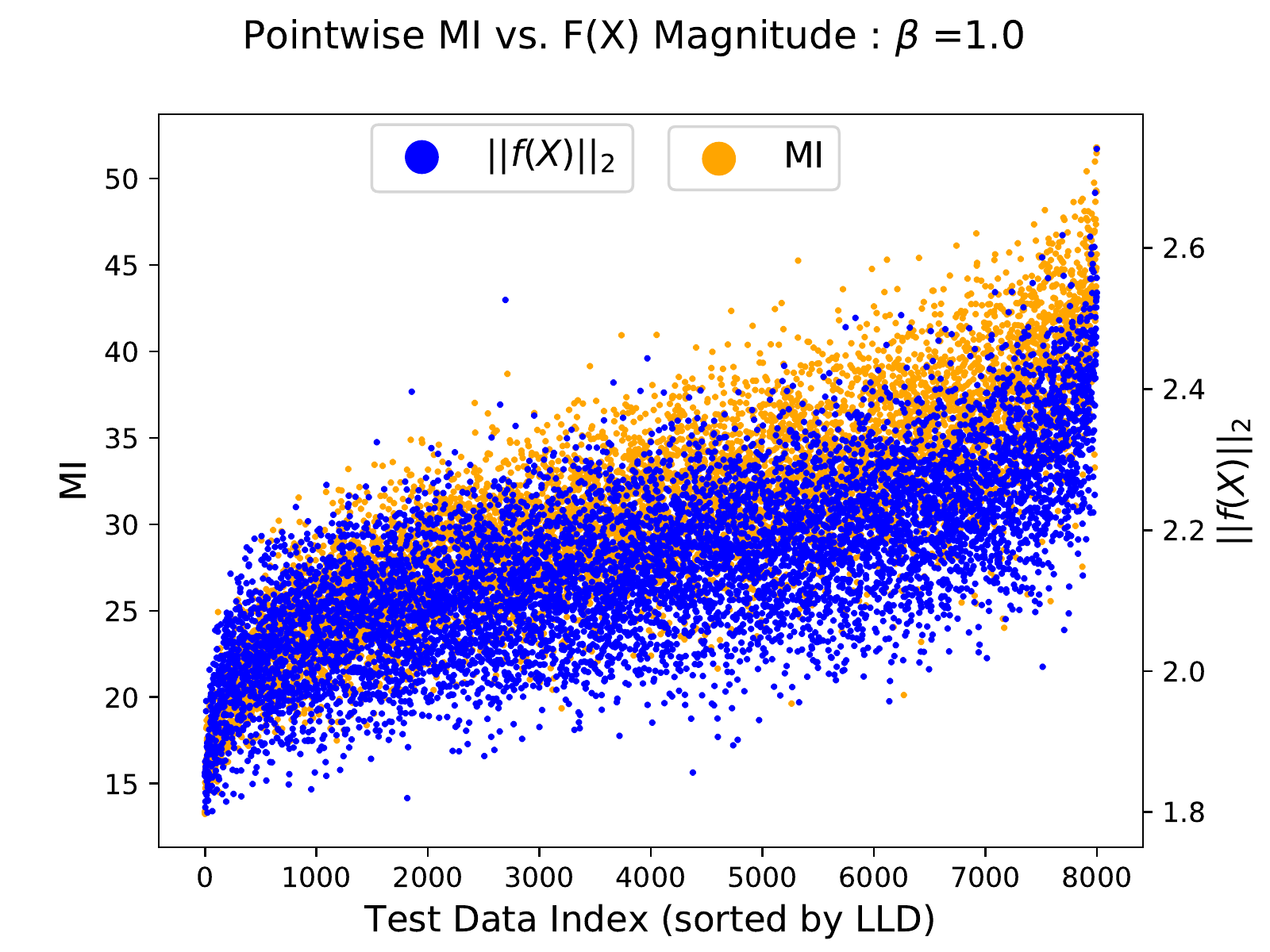}
    \captionof{figure}{Echo $f(\vx)$ vs. $S(\vx)$: Binary MNIST (left) and Omniglot (right)}
\end{minipage}


\section{Details for Experiments} \label{app:experiments}
All models were trained using a similar convolutional architecture as used in \cite{alemi2018fixing}, but with ReLU activations, unnormalized gradients, and fewer latent factors.  We use Keras notation and list convolutional layers using the arguments (filters, kernel size, stride, padding). We show an example parametrization of Echo in the hidden layer.  
\begin{itemize}
    \item Conv2D(32, 5, 1, `same')
    \item Conv2D(32, 5, 2, `same')
    \item Conv2D(64, 5, 1, `same')
    \item Conv2D(64, 5, 2, `same')
    \item Conv2D(256, 7, 1, `valid')
    \item \texttt{echo\_input} = [Dense(32, \texttt{tanh}($\cdot/16$)), \\
    Dense(32, \texttt{tf.math.log\_sigmoid})] 
    \item Lambda(\texttt{echo\_sample})(\texttt{echo\_input})
    \item Conv2DTranspose(64, 7, 1, `valid')
    \item Conv2DTranspose(64, 5, 1, `same')
    \item Conv2DTranspose(64, 5, 2, `same')
    \item Conv2DTranspose(32, 5, 1, `same')
    \item Conv2DTranspose(32, 5, 2, `same')
    \item Conv2DTranspose(32, 4, 1, `same')
    \item Conv2D(1, 4, 1, `same', activation = `sigmoid')
\end{itemize}
We trained using Adam optimization for 200 epochs, with a learning rate of 0.0003 decaying linearly to 0 over the last 100 epochs.   All experiments were run using NVIDIA Tesla V100 GPUs. 

MAF and IAF models were implemented using the Tensorflow Probability package \cite{dillon2017tensorflow}.  Each uses four steps of mean-only autoregressive flow, with each step consisting of three layers of 640 units.  \added{For the VampPrior, we used 750 pseudoinputs on all datasets.  For the \textit{IAF-Vamp} experiments, note that the VampPrior is calculated with respect to the inputs $\vz_0$ of the IAF transformation to avoid expensive density evaluations on new samples}.  This is valid since the mean-only transformation has constant Jacobian, but makes this method closely resemble VAE-Vamp.  All MMD penalties had a loss coefficient of 999, and were evaluated using a radial basis kernel with bandwidth $\sigma = 32/\sqrt{2}$ as in \cite{infovae, zhao2018lagrangian}. 


For rate-distortion experiments, we evaluated $\beta = [.05, .075, .1, .125, .15, .2, .25, .3, .4, .5, .6, .7, .8,$\\
$.9, 1, 1.5, 2, 3, 4, 6]$, with additional $\beta$ to fill in gaps in the curve as necessary. 

 For the disentanglement experiments in Sec. \ref{sec:disentanglement}, we followed the architecture and hyperparameters in \citet{locatello2018challenging}.  We trained for 300,000 gradient steps on both the full dataset and the downsampled dataset with dependent factors.  The visualization in Figure \ref{fig:disent_plot} was generated using code from \citet{chen2018isolating}.

Code implementing these experiments can be found at \small \texttt{https://github.com/brekelma/echo}.

\end{document}